\documentclass{article}

\usepackage{microtype}
\usepackage{graphicx}
\usepackage{subfigure}
\usepackage{booktabs} 
\usepackage{amsmath}
\usepackage{amsthm}
\usepackage{graphicx}
\usepackage{caption}
\usepackage{amsfonts}

\newtheorem{definition}{Definition}
\newtheorem{theorem}{Theorem}
\newtheorem{proposition}{Proposition}
\newtheorem{lemma}{Lemma}

\usepackage{hyperref}

\newcommand\independent{\protect\mathpalette{\protect\independenT}{\perp}}
\def\independenT#1#2{\mathrel{\rlap{$#1#2$}\mkern2mu{#1#2}}}
\usepackage[accepted]{icml2018}

\begin{document}

\twocolumn[
\icmltitle{Learning Optimal Policies from Observational Data }



\icmlsetsymbol{equal}{*}

\begin{icmlauthorlist}
\icmlauthor{Onur Atan}{to}
\icmlauthor{William R. Zame}{goo}
\icmlauthor{Mihaela van der Schaar}{ed,to}
\end{icmlauthorlist}

\icmlaffiliation{to}{Department of Electrical and Computer Engineering, University of California Los Angeles}
\icmlaffiliation{goo}{Department of Economics, University of California Los Angeles}
\icmlaffiliation{ed}{Department of Engineering, University of Oxford}

\icmlcorrespondingauthor{Onur Atan}{oatan@ucla.edu}

\icmlkeywords{Counterfactual Estimation, Domain Adaption, Neural Networks}

\vskip 0.3in
]



\printAffiliationsAndNotice{\icmlEqualContribution} 

\begin{abstract}
Choosing optimal (or at least better) policies is an important problem in  domains from medicine to education to finance and many others. One approach to this problem is through controlled experiments/trials - but controlled experiments are expensive.  Hence it is important to choose the best policies on the basis of observational data. This presents two difficult challenges: (i) missing counterfactuals, and (ii) selection bias. This paper presents theoretical bounds on estimation errors of counterfactuals from observational data by making connections to  domain adaptation theory. It also presents a principled way of choosing optimal policies using domain adversarial neural networks. We illustrate the effectiveness of domain adversarial training together with various features of our algorithm on a semi-synthetic breast cancer dataset and a supervised UCI dataset (Statlog). 
\end{abstract}

\section{Introduction}

The choice of a particular policy or plan of action involves consideration of the costs and benefits of the  policy/plan under consideration and also of alternative policies/plans that might be undertaken.  Examples abound; to mention just a few: Which course of treatment will lead to the most rapid recovery?  Which mode of advertisement will lead to the most orders?  Which investment strategy will lead to the greatest returns?  Obtaining information about the costs and benefits of alternative plans that might have been undertaken is a {\em counterfactual exercise}.  One possible way to estimate the counterfactual information is by conducting controlled experiments. However, controlled experiments are expensive,   involve small samples, and are frequently not available. It is therefore important to make decisions entirely on the basis of  observational data in which the actions/decisions taken in the data have been selected by an existing \textit{logging} policy.  Because the existing logging policy creates a selection bias, learning from  observational studies is a challenging problem.  This paper presents theoretical bounds on estimation errors for the evaluation of a new policy from observational data and a principled algorithm to learn the optimal policy. The methods and algorithms we develop are widely applicable (perhaps with some modifications) to an enormous range of settings, from healthcare to education to recommender systems to finance to smart cities.  (See   \cite{athey2015machine}, ~\cite{hoiles2016bounded} and \cite{bottou2013counterfactual}, for just a few examples.) 

As we have noted, our algorithm applies in many settings. In the medical context, features are the information included in electronic health records, actions are choices of different treatments, and outcomes are the success of treatment. In the financial context, features  are the aspects of the macroeconomic environment, actions are the choices of different investment decisions and outcomes are the revenues made by the investment decisions. In the recommender system context, features are the information about the user, the actions are choices of items, and outcomes are binary values indicating whether the user purchased the item or not. 

Our theoretical results show that {\em true policy outcome} is at least as good as the {\em policy outcome estimated from the observational data} minus the product of the number of actions with the $\mathcal{H}$-divergence between the observational and randomized data. Our theoretical bounds are different than ones derived in \cite{swaminathan15counterfactual} because ours do not require the propensity scores to be known. We use our theory to develop algorithm to learn balanced representations for each instance such that they are indistinguishable between the randomized and observational distribution and also predictive of the decision problem at hand. We present experiments on a semi-synthetic breast cancer and supervised Statlog data to show that our algorithm out-performs various methods, and to explain why.

\begin{table*}[t]
\normalsize
\centering
 
    \label{table:illustrate}
    \begin{tabular}{|c|c|c|c|c|}
    \hline
    Literature & Propensities known & Objective & Actions & Solution \\ \hline
     \cite{shalit2017estimating} & no & ITE estimation & $2$ & Balancing representations\\ \hline
     \cite{alaa2017bayesian} & no & ITE estimation & $2$ & Risk based empirical Bayes\\ \hline
     \cite{beygelzimer2009offset}& yes & policy optimization & $>2$ & Rejection sampling \\ \hline
      \cite{swaminathan15counterfactual,swaminathan2015self}& yes & policy optimization & $>2$ & IPS reweighing \\ \hline
     Ours & no & policy optimization & $>2$ & Balancing representations \\ \hline
    \end{tabular}
     \caption{Comparison with the related literature}
\end{table*}

\section{Related Work}
Roughly speaking, work on counterfactual learning from observational data falls into  two categories: estimation of ITEs ~\cite{johansson2016learning,shalit2017estimating,alaa2017bayesian} and Policy Optimization~\cite{swaminathan15counterfactual,swaminathan2015self}. The work on ITE's aims to estimate the expected  difference between outcomes for  \textit{treated} and \textit{control} patients, given the feature vector; this work focuses on settings with only two actions (treat/don't treat) - and notes that the approaches derived do not generalize well to settings with more than two actions. The work on policy optimization  aims to find a policy that maximizes the expected outcome (minimizes the risk). The policy optimization objective is somewhat easier than ITE objective in the sense that one can turn the ITE to action recommendations but not the other way around. In many applications, there are much more than $2$ actions; one is more interested in learning a good action rather than learning outcomes of each action for each instance. 

The work on ITE  estimation that is most closely related to ours focuses on learning balanced representations~\cite{johansson2016learning,shalit2017estimating}. These papers develop neural network algorithms to minimize the mean squared error  between predictions and actual outcomes in the observational data and also the discrepancy between the representations of the factual and counterfactual data. As these papers note, there is no principled approach to extend them to more than two treatments.  Other recent works in ITE estimation include tree-based methods~\cite{hill2011bayesian,athey2015machine,wager2015estimation} and Gaussian processes~\cite{alaa2017bayesian}. The last is perhaps the most successful, but  the computational complexity is $O(n^3)$ (where $n$  is the number of instances) so it is not easy to apply to large observational studies.

In the policy optimization literature, the work most closely related to ours is \cite{swaminathan15counterfactual,swaminathan2015self} where they develop the Counterfactual Risk Minimization (CRM) principle. The objective of the CRM principle is to minimize both the estimated mean and variance of the Inverse Propensity Score (IPS) instances; to do so the authors propose the POEM algorithm. Our work differs from POEM  in several ways: (i) POEM minimizes an objective over the  class of linear policies; we allow for  arbitrary  policies, (ii) POEM  requires the propensity scores to be available in the data; our  algorithm addresses the selection bias without using propensity scores, (iii) POEM  addresses selection bias by re-weighting each instance with the inverse propensities; our algorithm addresses the selection bias by learning representations. Another related paper on policy optimization is \cite{beygelzimer2009offset} which requires the propensity scores to be known and addresses the selection bias via rejection sampling. (For a more detailed comparison see  Table 1.)

The off-policy evaluation methods include IPS estimator~\cite{rosenbaum1983central, strehl2010learning}, self normalizing estimator~\cite{swaminathan2015self}, direct estimation, doubly robust estimator~\cite{dudik2011doubly, jiang2015doubly} and matching based methods~\cite{hill2006interval}. The IPS and self-normalizing estimators address the selection bias by re-weighting each instance by their inverse propensities.  The doubly robust estimation techniques combine the direct and IPS methods and generate more robust counterfactual estimates. Propensity Score Matching (PSM) replaces the missing counterfactual outcomes of the instance by the outcome of an instance with the closest propensity score. 

Our theoretical bounds have strong connection with the domain adaptation bounds given in~\cite{ben2007analysis,blitzer2008learning}. In particular, we show that the expected policy outcome is  bounded below by the estimate of the policy outcome from the observational data minus the product of the number of actions with the $\mathcal{H}$-divergence between the observational and randomized data. Our algorithm is based on domain adaptation as in~\cite{ganin2016domain}. Other domain adaptation techniques include \cite{zhang2013domain,daume2009frustratingly}.

\section{Problem Setup}
In this Section, we describe our formal model. 
\subsection{Observational Data}

We denote by $\mathcal{A}$  the set of $k$ actions, by $\mathcal{X}$ the $s$-dimensional space of features and by 
$\mathcal{Y} \subseteq R$  the space of outcomes.  We assume that an outcome can be identified with a real number and  normalize so that outcomes lie in the interval $\left[0,1\right]$. In some cases, the outcome will be either $1$ or $0$ (success or failure); in other cases the outcome may be interpreted as the probability of success or failure. We follow the potential outcome model described in the Rubin-Neyman causal model~\cite{rubin2005causal}; that is, for each instance $x \in \mathcal{X}$, there are $k$-potential outcomes: $Y^{(0)}, Y^{(1)}, \ldots, Y^{(k-1)} \in \mathcal{Y}$, corresponding to the $k$ different actions. The fundamental problem in this setting is that only the outcome of the action {\em actually performed} is recorded in the data: $Y = Y^{T}$. (This is called \textit{bandit feedback} in the machine learning literature~\cite{swaminathan15counterfactual}.) In our work, we focus on  the setting in which the action assignment is {\em not} independent of the feature vector, i.e., $A \not\independent X$; that is, action assignments are {\em not random}. This dependence is modeled by the conditional distribution $\gamma(a,x) = P(A = a| X = x)$, also known as the \textit{propensity score}. 

In this paper, we make the following common assumptions: 
\begin{itemize}
\item \textbf{Unconfoundedness:} Potential outcomes $(Y^{(0)}, Y^{(1)}, \ldots, Y^{(k-1)})$ are independent of the action assignment given the features, that is $(Y^{(0)}, Y^{(1)}, \ldots, Y^{(k-1)}) \independent A | X$. 
\item \textbf{Overlap:} For each instance $x \in \mathcal{X}$ and each action $a \in {\mathcal A}$, there is a non-zero probability that a patient with feature $x$ received the action $a$:  $0 < \gamma(a,x) <1$ for all $a,x$.
\end{itemize}
These assumptions are sufficient  to identify the optimal policy from the data~\cite{imbens2009recent,pearl2017detecting}.

We are given a data set
$$
\mathcal{D}^n = \{ (x_i, a_i, y_i) \}_{i=1}^n
$$ 
where each instance $i$ is generated by the following stochastic process:
\begin{itemize}
\item Each feature-action pair is drawn according to a fixed but unknown distribution $\mathcal{D}_S$, i.e, $(x_i, a_i) \sim \mathcal{D}_S$. 
\item Potential outcomes conditional on features are drawn with respect to a distribution $\mathcal{P}$; that is, $(Y_i^{(0)}, Y_i^{(1)}, \ldots, Y_i^{(k-1)}) \sim \mathcal{P}(\cdot| X = x_i, A = a_i)$. 
\item Only the outcome of the action actually performed is recorded in the data, that is, $y_i = Y_i^{(a_i)}$. 
\end{itemize}
We denote the  marginal distribution on the features by $\mathcal{D}$; i.e., $\mathcal{D}(x) = \sum_{a \in \mathcal{A}} \mathcal{D}_S(x,a)$.

\subsection{Definition of Policy Outcome}
A {\em policy} is a mapping  $h$ from features to actions. In this paper, we are interested in learning a policy $h$ that maximizes the policy outcome, defined as:
$$
V(h) = \mathbb{E}_{x \sim \mathcal{D}} \left[ \mathbb{E} \left[ Y^{(h(X))}| X = x\right] \right]. 
$$
We  denote by $m_a(x) = \mathbb{E}\left[ Y^{(a)} | X = x\right]$  the expected outcome of action $a$ on an instance with feature $x$. Based on these definitions, we can re-write the policy outcome of $h$ as  $V(h) = \mathbb{E}_{x \sim \mathcal{D}}\left[ m_{h(x)}(x) \right]$. Estimating $V(h)$ from the data is a challenging task because the counterfactuals are missing and there is a selection bias. 

\section{Counterfactual Estimation Bounds}
In this section, we provide a criterion that we will use to learn a policy $h^{*}$ the maximizes the outcome. We handle the selection bias in our dataset by mapping the features to representations are relevant to policy outcomes and are less biased. Let $\Phi: \mathcal{X} \rightarrow \mathcal{Z}$ denote a representation function which maps the features to representations. The representation function induces a distribution over representations $\mathcal{Z}$ (denoted by $\mathcal{D}^{\Phi}$) and $m_a$ as follows:
\begin{eqnarray*}
\mathbb{P}_{\mathcal{D}^{\Phi}}(\mathcal{B}) &=& \mathbb{P}_{\mathcal{D}}(\Phi^{-1}(\mathcal{B})), \notag \\ 
m_a^{\Phi}(z) &=& \mathbb{E}_{x \sim \mathcal{D}}[m_a(x) | \Phi(x) = z],
\end{eqnarray*} 
for any $\mathcal{B} \subset \mathcal{Z}$ such that $\Phi^{-1}(\mathcal{B})$ is $\mathcal{D}$-measurable. That is, the probability of of an event $\mathcal{B}$ according to $\mathcal{D}^{\Phi}$ is the probability of the inverse image of the event $\mathcal{B}$ according to $\mathcal{D}$. Our learning setting is defined by our choice of the representation function and hypothesis class $\mathcal{H} = \{h: \mathcal{Z} \rightarrow \mathcal{A} \}$ of (deterministic) policies. 

We now connect our problem to domain adaptation. Recall that $\mathcal{D}_S$ is the source distribution that generated feature-action samples in our observational data. Define the target distribution  $\mathcal{D}_{T}$ by $\mathcal{D}_{T}(x, a) = (1/K) \mathcal{D}(x)$. Note that $\mathcal{D}_S$ represents an observational study in which the actions are not randomized, while 
$\mathcal{D}_T$ represents a clinical study in which actions {\em are} randomized. Let $\mathcal{D}_S^{\Phi}$ and $\mathcal{D}_{T}^{\Phi}$ denote the source and target distributions induced by the representation function $\Phi$ over the space $\mathcal{Z} \times \mathcal{A}$, respectively.  Let  $\mathcal{D}^{\Phi}$ denote the marginal distribution over the representations and write $V^{\Phi}(h)$ for the induced policy outcome of $h$, that is, $V^{\Phi}(h) = \mathbb{E}_{z \sim \mathcal{D}^{\Phi}}\left[m_{h(z)}^{\Phi}(z)\right]$.

For the remainder of the theoretical analysis, suppose that the representation function $\Phi$ is fixed. The missing counterfactual outcomes can be addressed by importance sampling. Let $V_S^{\Phi}(h)$ and $V_T^{\Phi}(h)$ denote the expected policy outcome with respect to distributions $\mathcal{D}_S$ and $\mathcal{D}_T$, respectively. They are given by
\begin{eqnarray*}
V_S^{\Phi}(h) &=& \mathbb{E}_{(z,a) \sim \mathcal{D}_S^{\Phi}} \left[ \frac{ m_a^{\Phi}(z) 1(h(z) = a)}{1/k}\right], \\ 
V_T^{\Phi}(h) &=& \mathbb{E}_{(z,a) \sim \mathcal{D}_T^{\Phi}} \left[ \frac{ m_a^{\Phi}(z) 1(h(z) = a)}{1/k}\right].
\end{eqnarray*}
where $1(\cdot)$ is an indicator function if the statement is true and $0$ otherwise . We can only estimate $V_S(h)$ from the observational data. First, we'll connect $V_T^{\Phi}(h)$ with $V^{\Phi}(h)$, and provide some theoretical bounds based on the distance between source and target distribution.

\begin{proposition} Let $\Phi$ be a fixed representation function. Then: $V_T^{\Phi}(h) = V^{\Phi}(h)$.
\end{proposition}
\begin{proof} 
It follows that
\begin{eqnarray*}
V^{\Phi}_T(h) &=& \mathbb{E}_{z \sim \mathcal{D}^{\Phi}} \left[ \sum_{a \in \mathcal{A}} 1/k \frac{m_a^{\Phi}(z) 1(h(z) = a)}{1/k} \right] \\
&=& \mathbb{E}_{z \sim \mathcal{D}^{\Phi} } \left[ m_{h(z)}^{\Phi}(z) \right] = V^{\Phi}(h).
\end{eqnarray*}
\end{proof}

We can not create a Monte-Carlo estimator for $V^{\Phi}_T(h)$ since we don't have samples from the target distribution - we only have samples from the source distribution. Hence, we'll use  domain adaptation theory to bound the difference between $V_S^{\Phi}(h)$ and $V^{\Phi}_T(h)$ in terms of $\mathcal{H}$-divergence. In order to do that, we first need to introduce a distance metric between distributions. For any policy $h \in \mathcal{H}$, let $\mathcal{I}_h$ denote the characteristic set that contains all representation-action pairs that is mapped to label $a$ under function $h$, i.e., $\mathcal{I}_h = \{ (z,a): h(z) = a\}$. 

\begin{definition} Suppose $\mathcal{D}$, $\mathcal{D}'$ be probability distributions over $\mathcal{Z} \times \mathcal{A}$ such that every characteristic set $\mathcal{I}_h$ of $h \in \mathcal{H}$ is measurable with respect to both distributions. Then, the $\mathcal{H}$ -divergence between distributions $\mathcal{D}$ and $\mathcal{D}'$ is
$$
d_{\mathcal{H}}(\mathcal{D}, \mathcal{D}') = \sup_{h \in \mathcal{H}} \left| \mathbb{P}_{(z,a) \sim \mathcal{D}}(\mathcal{I}_h) - \mathbb{P}_{(z,a) \sim \mathcal{D}^{'}}(\mathcal{I}_h) \right|.
$$
\end{definition}

The $\mathcal{H}$-divergence measures the difference between the  behavior of policies in $\mathcal{H}$ when examples are drawn from    $\mathcal{D}$, $\mathcal{D}'$; this plays an important role in theoretical bounds. In the next lemma, we establish a bound on the difference between $V_S^{\Phi}(h)$ and $V^{\Phi}_T(h)$ based on the $\mathcal{H}$-divergence between source and target. 

\begin{lemma} Let $h \in \mathcal{H}$ and let $\Phi$ be a representation function. Then
$$
V^{\Phi}(h) \geq V_S^{\Phi}(h) - k d_{\mathcal{H}}(\mathcal{D}_T^{\Phi}, \mathcal{D}_S^{\Phi})
$$
\end{lemma}
\begin{proof} The proof is similar to \cite{ben2007analysis,blitzer2008learning}. The following inequality holds: 
\begin{eqnarray*}
V_S^{\Phi}(h) &=& \mathbb{E}_{(z,a) \sim \mathcal{D}_S^{\Phi}} \left[ \frac{m_a(z)}{1/k} 1(h(z) = a) \right]  \\
& \leq & \mathbb{E}_{(z,a) \sim \mathcal{D}_T^{\Phi}} \left[ \frac{m_a(z)}{1/k} 1(h(z) = a) \right] \\ 
&\;\;\;\;\;\;\;+&  k \left| \mathbb{P}_{(z,a) \sim \mathcal{D}_T^{\Phi}}(\mathcal{I}_h) - \mathbb{P}_{(z,a) \sim \mathcal{D}_{S}^{\Phi}}(\mathcal{I}_h) \right| \\
&\leq& V^{\Phi}(h) + k d_{\mathcal{H}}(\mathcal{D}_{S}^{\Phi}, \mathcal{D}_T^{\Phi})
\end{eqnarray*}
where the first inequality holds because $\frac{m_a(z)}{1/k} \leq k$ for all pairs  $(z,a)$ and  outcomes lie in the interval 
$\left[0,1\right]$. 
\end{proof}

Lemma 1 shows that the true policy outcome is at least as good as the policy outcome in the observational data minus the product of the number of actions times the $\mathcal{H}$-divergence between the observational and randomized data. (So, if the divergence is small, a policy that is found to be good with respect to the observational data is guaranteed to  be a good policy with respect to the true distribution.) We  create a Monte Carlo estimator $V_S^{\Phi}(h)$ for the policy outcome in source data and then use the lower bound we have just established to find the best action recommendation policy. 
 
\begin{definition} Let $\Phi$ be a representation function such that $\Phi(x_i) = z_i$. The {\em Monte-Carlo estimator} for the policy outcome in source data is given by: 
$$
\widehat{V}_S^{\Phi}(h)= \frac{1}{n} \sum_{i=1}^n \frac{y_i 1(h(z_i) = a_i)}{1/K}.
$$
\end{definition}

In order to provide uniform bounds on the Monte-Carlo estimator for an infinitely large class of recommendation functions, we need to first define a complexity term for a class $\mathcal{H}$. For $\epsilon > 0$, a policy class $\mathcal{H}$ and integer $n$, the growth function is defined as
$$
\mathcal{N}_{\infty}(\epsilon, \mathcal{H}, n) = \sup_{\boldsymbol{z} \in \boldsymbol{Z}^n} \mathcal{N}(\epsilon, \mathcal{H}(\boldsymbol{z}), \|\cdot \|_{\infty}),
$$
where $\mathcal{H}(\boldsymbol{z}) =\{\left(h(z_1), \ldots, h(z_n)\right): h\in\mathcal{H}\} \subset \mathbb{R}^n$, $\boldsymbol{Z}^n$ is the set of all possible $n$ representations and for $\mathcal{A} \subset \mathbb{R}^n$ the number $\mathcal{N}(\epsilon, A, \|\cdot\|_{\infty})$  is the  cardinality $|\mathcal{A}_0|$ of the smallest set  $\mathcal{A}_0 \subseteq \mathcal{A}$ such that $A$ is contained in the union of $\epsilon$-balls centered at points in $\mathcal{A}_0$ in the metric induced by $\|\cdot\|_{\infty}$. (This is often called the covering number.) Set $\mathcal{M}(n) = 10 \mathcal{N}_{\infty}(1/n, \mathcal{H}, 2n)$. The following result provides an inequality between the estimated and true $V_S^{\Phi}(h)$ for all $h \in \mathcal{H}$. 

\begin{lemma} \cite{maurer2009empirical} Fix $\delta \in \left(0,1\right)$, $n \geq 16$. Then, with probability $1 - \delta$, we have for all $h \in \mathcal{H}$:
$$
V_S^{\Phi}(h) \geq \widehat{V}_S^{\Phi}(h) - \sqrt{\frac{18 \ln(\mathcal{M}(n)/\delta)}{n}} - \frac{15\ln(\mathcal{M}(n)/\delta)}{n}
$$
\end{lemma}

In order to provide a data dependent bound on the estimation error between $V(h)$ and $\widehat{V}_S(h)$, we need to provide data-dependent bounds on the $\mathcal{H}$-divergence between source and target distributions. However, we aren't given samples from the target data so we need to generate (random) target data. Let $\widehat{\mathcal{D}}_S^{\Phi} = \{ (Z_i, A_i) \}_{i=1}^n$ denote the empirical distribution of the source data. From the empirical source distribution, we can generate target data by simply sampling the actions uniformly, that is, $\widehat{\mathcal{D}}_T^{\Phi} = \{ (Z_i, \widetilde{A}_i)\}_{i=1}^n$ where $\widetilde{A}_i \sim \operatorname{Multinomial}(\left[1/K, \ldots, 1/K \right])$. Then, we have $\widehat{\mathcal{D}}_S^{\Phi} \sim \mathcal{D}_S^{\Phi}$ and $\widehat{\mathcal{D}}_T^{\Phi} \sim \mathcal{D}_T^{\Phi}$. Then, define the empirical probability estimates of the characteristic functions as 
\begin{eqnarray*}
\mathbb{P}_{(z,a) \sim \widehat{\mathcal{D}}_S^{\Phi}}(\mathcal{I}_h) &=&  \frac{1}{n}\sum_{i=1}^n 1(h(Z_i) = A_i), \\
\mathbb{P}_{(z,a) \sim \widehat{\mathcal{D}}_T^{\Phi}}(\mathcal{I}_h) &=&  \frac{1}{n}\sum_{i=1}^n 1(h(Z_i) = \tilde{A}_i).
\end{eqnarray*}
Then, one can compute empirical $\mathcal{H}$-divergence between two samples $\widehat{\mathcal{D}}_S^{\Phi}$ and $\widehat{\mathcal{D}}_T^{\Phi}$ by
\begin{equation}
d_{\mathcal{H}}(\widehat{\mathcal{D}}_T^{\Phi}, \widehat{\mathcal{D}}_S^{\Phi}) = \sup_{h \in \mathcal{H}} \left| \mathbb{P}_{(z,a) \sim \widehat{\mathcal{D}}_T^{\Phi}}(\mathcal{I}_h) - \mathbb{P}_{(z,a) \sim \widehat{\mathcal{D}}_S^{\Phi}}(\mathcal{I}_h) \right|.
\end{equation}
In the next lemma, we provide estimation bounds between the empirical $\mathcal{H}$-divergence and true $\mathcal{H}$-divergence. 

\begin{lemma}  Fix $\delta \in \left(0,1\right)$, $n \geq 16$. Then, with probability $1 - 2\delta$, we have for all $h \in \mathcal{H}$: 
\begin{eqnarray*}
d_{\mathcal{H}}(\mathcal{D}_T^{\Phi}, \mathcal{D}_S^{\Phi}) &\geq& d_{\mathcal{H}}(\widehat{\mathcal{D}}_T^{\Phi}, \widehat{\mathcal{D}}_S^{\Phi}) \\ 
&-& 2\left[\sqrt{\frac{18 \ln(\mathcal{M}(n)/\delta)}{n}} - \frac{15\ln(\mathcal{M}(n)/\delta)}{n}\right]
\end{eqnarray*}
\end{lemma}
\begin{proof}
Define 
$
\beta(\delta, n) = \sqrt{\frac{18 \ln(\mathcal{M}(n)/\delta)}{n}} - \frac{15\ln(\mathcal{M}(n)/\delta)}{n}. 
$
By \cite{maurer2009empirical}, with probability $1 - \delta$, we have for each hypothesis $h \in \mathcal{H}$,
\begin{eqnarray*}
\mathbb{P}_{(z,a) \sim \mathcal{D}_T^{\Phi}}(\mathcal{I}_h) &\geq& \mathbb{P}_{(z,a) \sim \widehat{\mathcal{D}}_T^{\Phi}}(\mathcal{I}_h) - \beta(\delta, n) \\ 
\mathbb{P}_{(z,a) \sim \mathcal{D}_S^{\Phi}}(\mathcal{I}_h) &\leq& \mathbb{P}_{(z,a) \sim \widehat{\mathcal{D}}_S^{\Phi}}(\mathcal{I}_h)  + \beta(\delta, n)
\end{eqnarray*}
Hence, by union bound, the following equation holds for all $h \in \mathcal{H}$ with probability $1 - 2 \delta$:
\begin{multline*}
\left|\mathbb{P}_{(z,a) \sim \mathcal{D}_T^{\Phi}}(\mathcal{I}_h) - \mathbb{P}_{(z,a) \sim \mathcal{D}_S^{\Phi}}(\mathcal{I}_h)\right| \\
\geq \left|\mathbb{P}_{(z,a) \sim \widehat{\mathcal{D}}_T^{\Phi}}(\mathcal{I}_h) - \mathbb{P}_{(z,a) \sim \widehat{\mathcal{D}}_S^{\Phi}}(\mathcal{I}_h) - 2\beta(\delta,n)\right| 
\end{multline*}
The inequality still holds by taking supremum over $\mathcal{H}$ with $1 - 2 \delta$, that is, 
\begin{eqnarray*}
&&d_{\mathcal{H}}(\mathcal{D}_T^{\Phi}, \mathcal{D}_S^{\Phi}) \\
&\geq& \sup_{h \in \mathcal{H}} \left|\mathbb{P}_{(z,a) \sim \widehat{\mathcal{D}}_T^{\Phi}}(\mathcal{I}_h) - \mathbb{P}_{(z,a) \sim \widehat{\mathcal{D}}_S^{\Phi}}(\mathcal{I}_h) - 2\beta(\delta,n)\right|  \\ 
&\geq& d_{\mathcal{H}}(\widehat{\mathcal{D}}_T^{\Phi}, \widehat{\mathcal{D}}_S^{\Phi}) - 2 \beta(\delta, n).
\end{eqnarray*}
where the last inequality follows  from the triangle inequality.
\end{proof}

Finally, by combining Lemmas 1,2 and 3, we obtain a data-dependent bound on the counterfactual estimation error. 

\begin{theorem} 
Fix $\delta \in (0,1)$, $n \geq 16$. Let $\Phi$  be the representation function and let $\mathcal{H}$ be the set of policies. Then, with probability at least $1 - 3\delta$, we have for all $h \in \mathcal{H}$: 
\begin{eqnarray*}
V^{\Phi}(h) &\geq& \widehat{V}_S^{\Phi}(h) - k d_{\mathcal{H}}(\widehat{\mathcal{D}}_S^{\Phi}, \widehat{\mathcal{D}}_T^{\Phi}) \\ 
&-& 3k \bigg[\sqrt{\frac{18\ln(\mathcal{M}(n)/\delta)}{n}} - \frac{15\ln(\mathcal{M}(n)/\delta)}{n} \bigg]
\end{eqnarray*}
\end{theorem}

The result provided in Theorem 1 is constructive and motivates our optimization criteria. 

\begin{figure}[h]
\includegraphics[width = 0.5\textwidth]{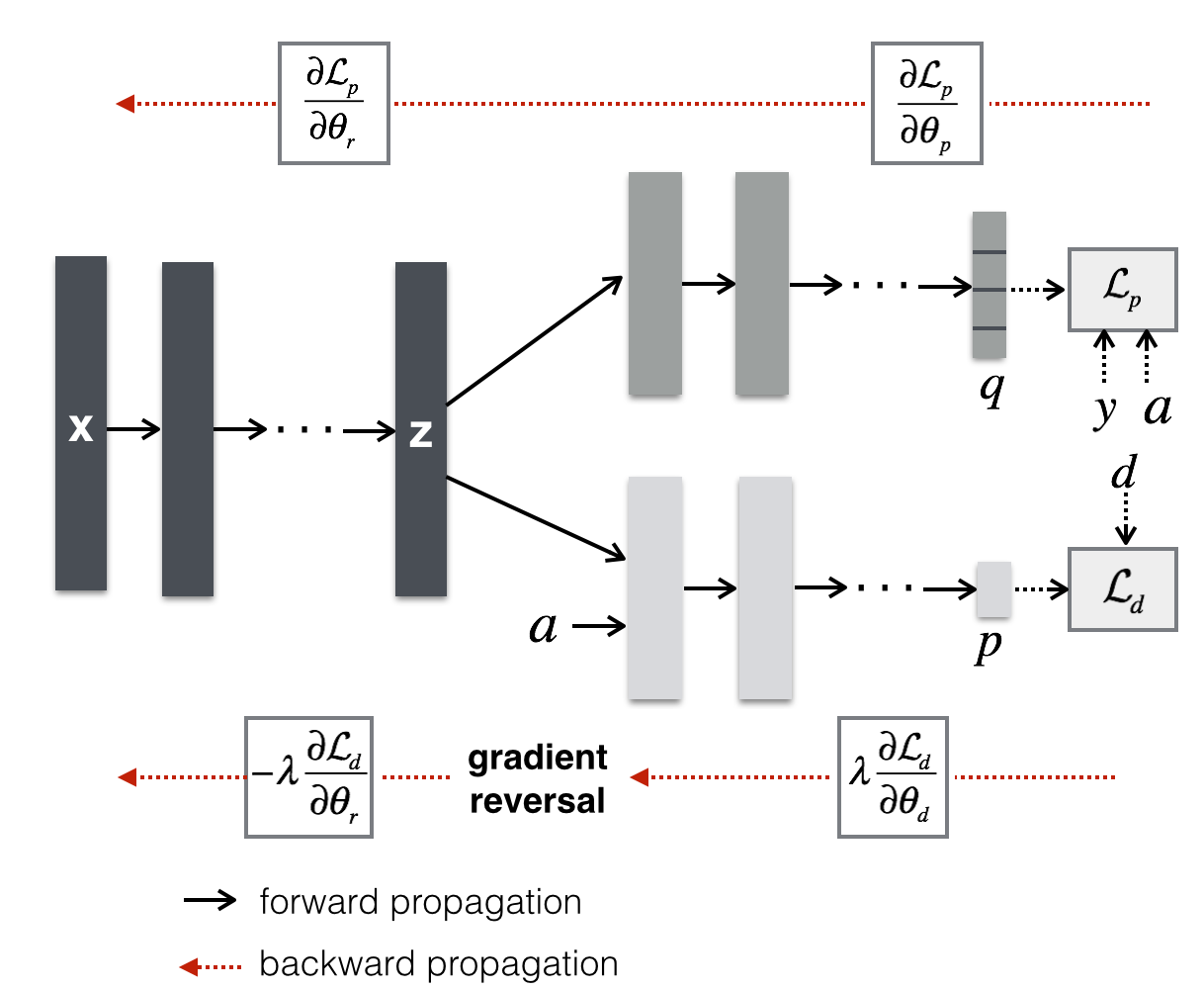}
\label{fig:sigma}
\caption{Neural network model based on~\cite{ganin2016domain}}
\end{figure}

\section{Counterfactual Policy Optimization (CPO)}
Theorem 1 motivates a general framework for designing policy learning from observational data with bandit feedback. A learning algorithm following this criterion  solves:
$$
\widehat{\Phi}, \widehat{h} = \arg\max_{\Phi, h} \; \widehat{V}_S^{\Phi}(h) - \lambda d_{\mathcal{H}}(\widehat{\mathcal{D}}_S^{\Phi}, \widehat{\mathcal{D}}_T^{\Phi}),
$$
where $\lambda > 0$ is the trade-off parameter between the empirical policy outcome in the source data and the empirical $\mathcal{H}$-divergence between the source and target distributions. This optimization criterion  seeks to find a representation function where  the source and the target domain are indistinguishable. Computing the empirical $\mathcal{H}$-divergence between the source and target distributions is known to be NP-hard ~\cite{ganin2016domain}, but we can use recent developments in  domain adversarial neural networks to find a good approximation.  

\subsection{Domain Adversarial Neural Networks}
In this paper, we follow the recent work in domain adversarial training of  neural networks~\cite{ganin2016domain}. For this, we need samples from observed data - sometimes referred to as source data ($\mathcal{D}_S$) - and unlabeled samples from an ideal dataset - referred to as target data ($\mathcal{D}_T$). As mentioned, we don't have samples from an ideal dataset. Hence, we'll first talk about batch sampling of source and target from our dataset $\mathcal{D}$. Given a batch size of $m$, we randomly sample from $\mathcal{D}$ and set domain variable $d = 0$ indicating this is the source data. Then, we sample $m$ additional samples excluding the samples from the source data and randomly assign an action according to the distribution 
$\operatorname{Multinomial}(\left[1/K, \ldots, 1/K\right])$;  finally, we set the domain variable $d = 1$ indicating this is the target data. The batch generation procedure is depicted in Algorithm 1. 

\begin{algorithm}[t]
\caption{Procedure: $\text{Generate}-\text{Batch}$}
	\label{alg:gen_batch}
	\normalsize
	\begin{algorithmic}[1]
		\STATE Input: Data: $\mathcal{D}_n$, Batch size: $m$
		\STATE Sample $\mathcal{U} = \{u_1, \ldots, u_m\} \subset \mathcal{N} =\{1,\ldots,n \}$.
		\STATE Set source set $\mathcal{S} = \{ (x_{u_i}, a_{u_i}, y_{u_i}, d_{i} =0) \}_{i=1}^m$.
		\STATE Sample $\mathcal{V} =\{ v_1, \ldots, v_m\} \subset \mathcal{N} \setminus \mathcal{U}$. 
		\STATE Set $\mathcal{T} = \emptyset$
		\FOR{i = 1, \ldots, m:}
		\STATE Sample $\widetilde{a}_i \sim \operatorname{Multinomial}([1/K, \ldots, 1/K])$. 
		\STATE $\mathcal{T} = \mathcal{T} \cup \{ (x_{v_i}, \widetilde{a}_i, d_i = 1) \}$. 
		\ENDFOR
		\STATE Output: $\mathcal{S}$, $\mathcal{T}$. 
	\end{algorithmic}
\end{algorithm}

Our algorithm consists of three blocks: representation, domain and policy blocks. In the representation block, we seek to find a map $\Phi: \mathcal{X} \rightarrow \mathcal{Z}$ combining two objectives: (i) high predictive power on the outcomes, (ii) low predictive power on the domain. Let $F_r$ denote a parametric function that maps the patient features to representations, that is, $z_i = F_r(x_i; \theta_r)$ where $\theta_r$ is the parameter vector of the representation block. The representations are input to both survival and policy blocks. Let $F_p$ denote the mappings from representation-action pair $(z_i, a_i)$ to probabilities over the actions $\hat{q}_i = \left[ \hat{q}_{i,0}, \ldots, \hat{q}_{i, K-1}\right]$, i.e., $\hat{q}_i = F_p(z_i, a_i; \theta_p)$ where $\theta_p$ is the parameter vector of the policy block. For an instance with features $x_i$ and action $a_i$, an element in output of policy block $\hat{q}_{i,a}$ is the probability of recommending action $a$ for subject $i$. The estimated policy outcome in source data is then given by
$$
\widehat{V}_{S}^{\Phi}(h) = \frac{1}{n} \sum_{i=1}^n \frac{y_i q_{a_i}}{1/k}.
$$
Although our theory applies only to deterministic policies, we will allow for stochastic policies in order to make the optimization problem tractable. This is not optimal; however, as we'll show in our numerical results, this approach is still able to achieve significant gains with respect to benchmark algorithms. Let $G_d$ be a mapping from representation-action pair $(z_i, a_i)$ to probability of the instances generated from target, i.e., $\hat{p}_i = G_d(z_i, a_i; \theta_d)$ where $\theta_d$ is the parameters of the domain block. 

Note that the last layer of the policy block is a softmax operation, which has exponential terms. Instead of directly maximizing 
$\widehat{V}_{S}(h)$, we use a modified cross-entropy loss to make the optimization criteria more robust. The policy loss is then 
$$
\mathcal{L}_p^i(\theta_r, \theta_s) = \frac{-y_i \log(q_{i, a_i})}{1/k}
$$
At the testing stage, we can then convert these probabilities to action recommendations simply by recommending the action with highest probability $q_{i,a}$. We set the domain loss to be the standard cross entropy loss between the estimated domain probability $p_i$ and the actual domain probability $d_i$; this is the standard classification loss used  in the literature and is given by 
$$
\mathcal{L}_d^i(\theta_r, \theta_s) = d_i \log(p_i) + (1 - d_i) \log p_i.
$$


Our goal in this paper to find the saddle point that optimizes the weighted sum of the survival and domain loss. This total loss is given by
\begin{eqnarray*}
\mathcal{E}(\theta_r, \theta_s, \theta_d) &=& \sum_{i \in \mathcal{S}} \mathcal{L}_s^i(\theta_r, \theta_s) \notag \\ 
&\;&\;\; - \lambda \left( \sum_{i \in \mathcal{S}} \mathcal{L}_d^i(\theta_r, \theta_d) + \sum_{i \in \mathcal{T}} \mathcal{L}_d^i(\theta_r, \theta_d) \right) 
\end{eqnarray*}
where $\lambda > 0$ is the trade-off between survival and domain loss. The saddle point is
\begin{eqnarray*}
\left(\hat{\theta}_{r}, \hat{\theta}_{s} \right) &=& \arg\min_{\theta_{r}, \theta_{p}} \; \mathcal{E}(\theta_{r}, \theta_p, \hat{\theta}_d), \notag \\ 
\hat{\theta}_d &=& \arg\max_{\theta_{d}} \; \mathcal{E}(\hat{\theta}_{r}, \hat{\theta}_p, \theta_d). 
\end{eqnarray*}
The training procedure of the Domain Adverse training of Counterfactual POLicy training (DACPOL)  is depicted in Algorithm 2. The neural network architecture is depicted in Figure 1. 

For a test instance with covariates $x^{*}$, we compute the action recommendations with the following procedure: We first compute the representations by $z^{*}= G_r(x^{*}; \theta_r)$, then compute the action probabilities $q^{*} = F_p(z^{*}, \theta_p)$. We finally recommend the action with $\hat{A}(x^{*}) = \arg\max_{a \in \mathcal{A}} q^{*}_{a}$.

\begin{algorithm}[t]
\caption{Training Procedure: $\text{DACPOL}$}
	\label{alg:gen_batch}
	\normalsize
	\begin{algorithmic}[1.5]
		\STATE Input: Data: $\mathcal{D}$, Batch size: $m$, Learning rate: $\mu$
		\STATE $(\mathcal{S}, \mathcal{T}) = \text{Generate-Batch}(\mathcal{D}, m)$. 
		\FOR{until convergence}
		\STATE Compute $\mathcal{L}_p^{\mathcal{S}}( \theta_{r}, \theta_s) = \frac{1}{|\mathcal{S}|} \sum_{i \in \mathcal{S}} \mathcal{L}_p^i( \theta_{r}, \theta_s)$
		\STATE Compute $\mathcal{L}_d^{\mathcal{S}}(\theta_{r}, \theta_d) = \frac{1}{|\mathcal{S}|} \sum_{i \in \mathcal{S}}\mathcal{L}_d^i(\theta_{r}, \theta_d)$
		\STATE Compute $\mathcal{L}_d^{\mathcal{T}}(\theta_{r}, \theta_d) = \frac{1}{|\mathcal{T}|} \sum_{i \in \mathcal{T}}\mathcal{L}_d^i(\theta_{r}, \theta_d)$
		\STATE Compute $\mathcal{L}_d(\theta_{r}, \theta_d) = \mathcal{L}_d^{\mathcal{S}}(\theta_{r}, \theta_d) + \mathcal{L}_d^{\mathcal{T}}(\theta_{r}, \theta_d)$
		\STATE $\theta_{r} \rightarrow \theta_{r} - \mu \left (\frac{\partial \mathcal{L}_s^{\mathcal{S}}( \theta_{r}, \theta_s)}{\partial \theta_{r}} -\lambda \frac{\partial \mathcal{L}_d( \theta_{r}, \theta_d)}{\partial \theta_{r}} \right)$
		\STATE $\theta_p \leftarrow \theta_p - \mu \frac{\partial \mathcal{L}_p^{\mathcal{S}}( \theta_{r}, \theta_s)}{\theta_s}$
		\STATE $\theta_d \leftarrow \theta_d - \mu \frac{\partial \mathcal{L}_d( \theta_{r}, \theta_d)}{\theta_d}$
		\ENDFOR
	\end{algorithmic}
\end{algorithm}

\section{Numerical Results}
Here we describe the performance of our algorithm. Note that it is difficult (almost impossible) to test and validate the algorithm on real data with missing counterfactual survival outcomes. In this paper, we provide results both on a semi-synthetic breast cancer and a supervised UCI dataset (Statlog). 

\subsection{Dataset Description}
\textbf{Breast cancer dataset: }
The dataset includes $10,000$ records of breast cancer patients participating in the National Surgical Adjuvant Breast and Bowel Project (NSABP); see \cite{yoon2016discovery}. Each instance consists of the following information about the patient: age, menopausal, race, estrogen receptor, progesterone receptor, human epidermal growth factor receptor 2 (HER2NEU), tumor stage, tumor grade, Positive Axillary Lymph Node Count(PLNC), WHO score, surgery type, Prior Chemotherapy, prior radiotherapy and histology. The treatment is a choice among six chemotherapy regimens of which only $5$ of them are used: AC, ACT, CAF, CEF, CMF. The outcomes for these regimens were derived based on 32 references from PubMed Clinical Queries. The data contains the feature vector $x$ and all derived outcomes for each treatment $\{ Y_t \}_{t \in \mathcal{T}}$. 

\textbf{UCI Statlog Dataset:}
This dataset includes the multi-spectral values of pixels in a satellite image. The feature vector contains $36$ pixel values and the aim is to predict the true description of the plot (barren soil, grass, cotton crop, etc.) from the feature vector. We follow the same procedure summarized in \cite{beygelzimer2009offset}. That is, we treat each label as an action and set the outcome of the action which matches with the label as $1$ and the rest as $0$.

\vspace{-.5em}

\subsection{Experimental Setup}
We generate an artificially biased dataset $\mathcal{D}^n =\{ (X_i, A_i, Y_i)\}$ by the following procedure: (i) we first draw  random weights $W \in \mathbb{R}^{s \times k}$ with $w_{j,a} \sim \mathcal{N}(0, \sigma I)$ where $\sigma > 0$ is a parameter used to generate datasets with different selection bias levels. We generate actions in the data according to the logistic distribution $A \sim \exp(x^T w_a)/ (\sum_{a \in \mathcal{A}} \exp(x^T w_a))$.

For the breast cancer data set, we generate a $56/24/20$ split of the data to train, validate and test our DACPOL. For the Statlog data, we spare $30\%$ of the training data for validation and use the testing set provided to evaluate our algorithm. The hyperparameter list we used in our validation set is $10^\gamma/2$ with $\gamma \in \left[ -4, -3, -2, -1,  0, 0.5, 0.75, 1, 1.5, 2, 3\right]$. We generate $100$ different datasets by following the procedure described above and report the average and $95\%$ confidence levels.

The performance metric we use to evaluate our algorithm in this paper is loss, which we define to be $1 - \text{accuracy}$; accuracy is defined as the fraction of test instances in which the recommended and best action match.  Note that we can evaluate the accuracy metric since we have the ground truth outcomes in the testing set, but of course the ground truth outcomes are not used by any algorithm in the training and validation test. In our experiments, we use 1-1-2 representation/domain/outcome fully-connected layers. The neural network is trained by back propagation via Adam Optimizer~\cite{kingma2014adam}  with an initial learning rate of $.01$. We begin with an initial learning rate $\mu$ and tradeoff parameter $\lambda$ and  use iterative adaptive parameters to get our result; along the way we decrease the learning rate $\mu$ and increases the tradeoff parameter. This is standard procedure in training domain adversarial neural networks~\cite{ganin2016domain}. We implement DACPOL in the Tensorflow environment. 

\vspace{-.5em}

\subsection{Benchmarks} 
We compare performance of DACPOL with two benchmarks
\begin{itemize}
\item \textbf{POEM}~\cite{swaminathan15counterfactual}  is a linear policy optimization algorithm which minimizes the empirical risk of IPS estimator and variance. 
\item \textbf{IPS} is  POEM without  variance regularization. 
\end{itemize}
Both IPS and POEM deal with the selection bias in the data by using the propensity scores. Note that DACPOL does not require the propensity scores to be known in order to address the selection bias. Hence, in order to make fair comparisons, we estimate the propensity scores from the data, and use these estimates in IPS and POEM. 

\vspace{-.5em}

\subsection{Results}
\subsubsection{Comparisons with the benchmarks}
Table 2 shows the discriminative performance of DACPOL (in which we optimize $\lambda$) and DACPOL(0) (in which we set $\lambda = 0$) with the  benchmark algorithms. We compare the algorithms in breast cancer data with $\sigma = 0.3$ and Statlog data with $\sigma = 0.2$. As seen from the table, our algorithm outperforms the benchmark algorithms in terms of the loss metric defined above. The empirical gain with respect to POEM algorithm has three sources: (i) DACPOL does not need propensity scores, (ii) DACPOL optimizes over all policies not just  linear policies, (iii) DACPOL trades off between the predictive power and bias introduced by the features. (We further illustrate the last source of gain in the next subsection with a toy example.)

\begin{table}[h] 
\normalsize
\centering 
\begin{tabular}{|c|c|c|} 
\hline 
Algorithm & Breast Cancer & Statlog    \\ \hline
DACPOL & $.292 \pm .006$ & $.249 \pm .015$ \\ \hline
DACPOL(0)  & $.321 \pm .006$ & $.261 \pm  .020$ \\ \hline \hline
POEM & $.394 \pm .004$ & $.432 \pm .016$  \\ \hline
IPS & $.397 \pm .004$ & $.454 \pm .017$ \\ \hline
\end{tabular} 
\caption{Loss Comparisons for Breast Cancer and Statlog Dataset; Means and 95\% Confidence Intervals} 
\label{table:res_benchmarks}
\end{table}

\vspace{-1em}

\subsubsection{Domain Loss and Policy Loss}

The hyperparameter $\lambda$ controls the domain loss in the training procedure. As  $\lambda$ increases, the domain loss in training DACPOL increases; eventually source and target become indistinguishable, the representations become balanced, and  the loss of  DACPOL reaches a minimum.  If we increase $\lambda$ beyond that point, the algorithm classifies the source as the target and the target as the source,  representations become unbalanced, and  the the loss of DACPOL increases again.  Figure 2 illustrates this effect for the breast cancer dataset.  

\begin{figure}[h]
\includegraphics[width = 0.5\textwidth]{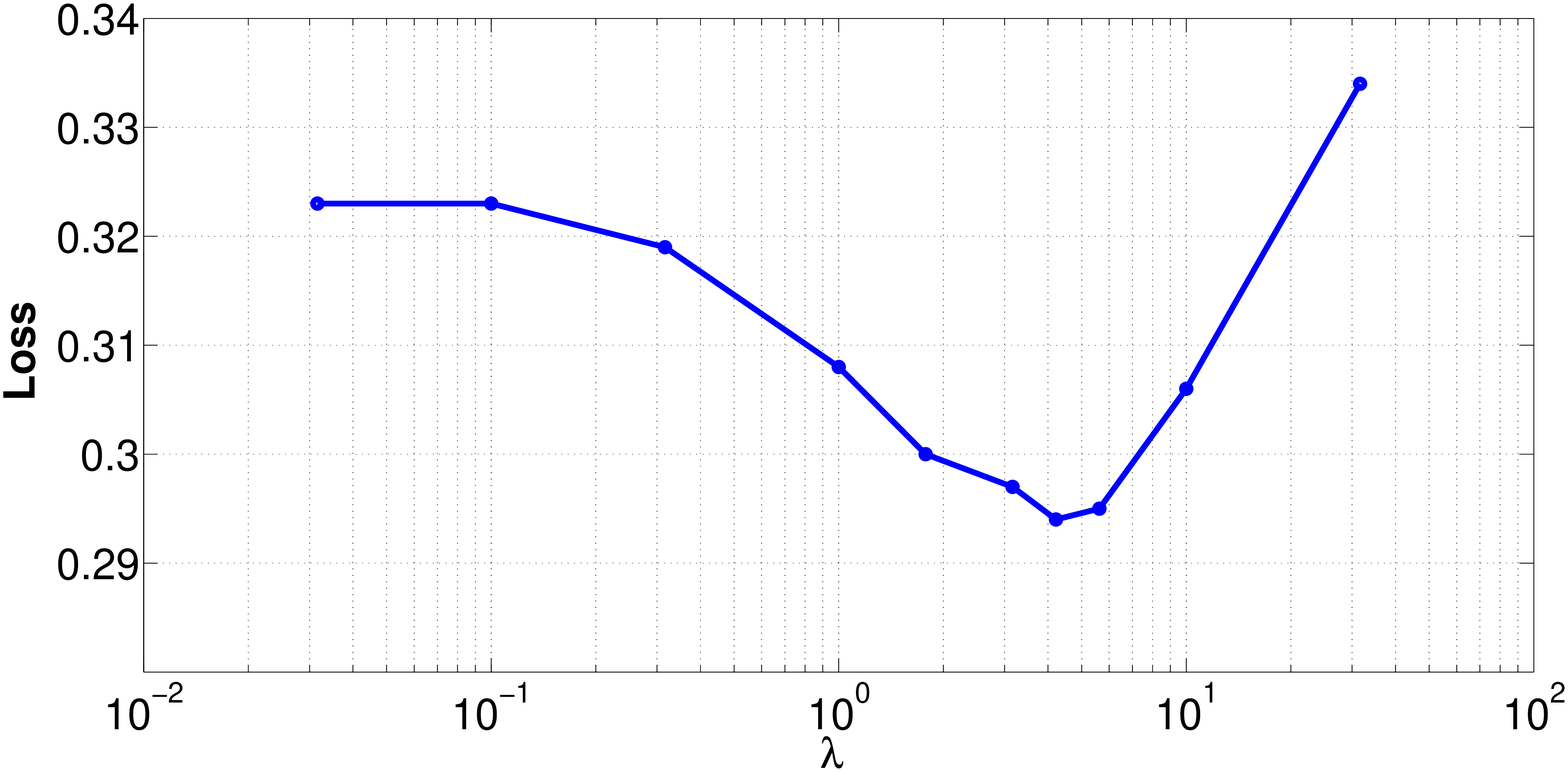}
\caption{The effect of domain loss in DACPOL performance}
\label{fig:sigma}
\end{figure}

\vspace{-1.2em}

\subsubsection{The effect of  selection bias in DACPOL}
In this subsection, we show the effect of the selection bias in the performance of our algorithm by varying the parameter 
$\sigma$ in our data generation process: a larger value of  $\sigma$ creates more biased data. Figure 3 shows two important points: (i) as the selection bias increases, the loss of  DACPOL increases, (ii) as the selection bias increases, domain adversarial training becomes more efficient, and hence the improvement of DACPOL over DACPOL(0) increases.

\begin{figure}[h]
\includegraphics[width = 0.5\textwidth]{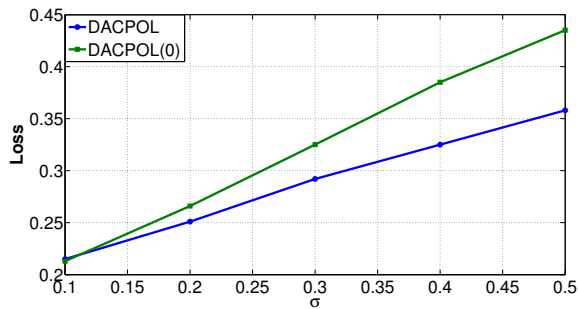}
\label{fig:sigma}
\caption{The effect of selection bias in DACPOL performance}
\end{figure}

\vspace{-0.5em}

\subsubsection{Empirical Gains with respect to CRM }

In this subsection, we show an advantage of our CPO principle over the CRM principle: selection bias from irrelevant features is less important. This happens because our representation optimization is able to remove the effect of the irrelevant features in the outputted representations and then uses only the relevant features to directly estimate the policy outcome as if it had access to randomized data. However, the performance of the CRM principle (whose objective is to maximize the IPS estimator minus the variance of the policy outcome) decreases with additional irrelevant features, because the inverse propensities due to irrelevant features  become large, and hence the variance of the IPS estimator will also become large. To see this, we use a toy example. We begin with 15 relevant features $x$.  We then generate $d$ additional irrelevant features $z \sim \mathcal{N}(0, I)$.  We create a logging policy that depends only on the irrelevant features using the logistic distribution. As $d$ increases, the selection bias also increases.  As Figure 4 shows, POEM is  more sensitive than DACPOL to this increase in the selection bias.  

\begin{figure}[h]
\includegraphics[width = 0.5\textwidth]{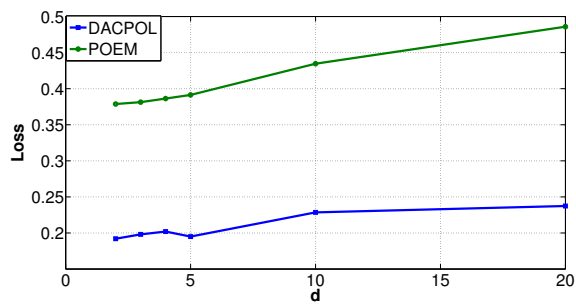}
\caption{The effect of irrelevant features in DACPOL vs POEM}
\end{figure}

\vspace{-2em}

\section{Conclusion} \vspace{-.5em}
This paper presented estimation bounds on the error between actual and estimated policy outcomes from observational data. Our theoretical results show that the estimation error from observational data depends on the $\mathcal{H}$-divergence between the observational and randomized data. This result motivated the development of a domain adversarial neural network to learn an optimal policy from observational data. We illustrated various features of our algorithm semi-synthetic and real data. Future work includes multi-stage actions, time-varying features etc. 

\bibliographystyle{icml2018}
\bibliography{NIPSrelevant}

\begin{thebibliography}{27}
\providecommand{\natexlab}[1]{#1}
\providecommand{\url}[1]{\texttt{#1}}
\expandafter\ifx\csname urlstyle\endcsname\relax
  \providecommand{\doi}[1]{doi: #1}\else
  \providecommand{\doi}{doi: \begingroup \urlstyle{rm}\Url}\fi

\bibitem[Alaa \& van~der Schaar(2017)Alaa and van~der Schaar]{alaa2017bayesian}
Alaa, Ahmed~M and van~der Schaar, Mihaela.
\newblock Bayesian inference of individualized treatment effects using
  multi-task gaussian processes.
\newblock In \emph{Advances in Neural Information Processing Systems (NIPS)},
  2017.

\bibitem[Athey \& Imbens(2015)Athey and Imbens]{athey2015machine}
Athey, Susan and Imbens, Guido~W.
\newblock Recursive partitioning for heterogeneous causal effects.
\newblock \emph{arXiv preprint arXiv:1504.01132}, 2015.

\bibitem[Ben-David et~al.(2007)Ben-David, Blitzer, Crammer, and
  Pereira]{ben2007analysis}
Ben-David, Shai, Blitzer, John, Crammer, Koby, and Pereira, Fernando.
\newblock Analysis of representations for domain adaptation.
\newblock In \emph{Advances in neural information processing systems}, pp.\
  137--144, 2007.

\bibitem[Beygelzimer \& Langford(2009)Beygelzimer and
  Langford]{beygelzimer2009offset}
Beygelzimer, Alina and Langford, John.
\newblock The offset tree for learning with partial labels.
\newblock In \emph{Proceedings of the 15th ACM SIGKDD international conference
  on Knowledge discovery and data mining}, pp.\  129--138, 2009.

\bibitem[Blitzer et~al.(2008)Blitzer, Crammer, Kulesza, Pereira, and
  Wortman]{blitzer2008learning}
Blitzer, John, Crammer, Koby, Kulesza, Alex, Pereira, Fernando, and Wortman,
  Jennifer.
\newblock Learning bounds for domain adaptation.
\newblock In \emph{Advances in neural information processing systems}, pp.\
  129--136, 2008.

\bibitem[Bottou et~al.(2013)Bottou, Peters, Candela, Charles, Chickering,
  Portugaly, Ray, Simard, and Snelson]{bottou2013counterfactual}
Bottou, L{\'e}on, Peters, Jonas, Candela, Joaquin~Quinonero, Charles,
  Denis~Xavier, Chickering, Max, Portugaly, Elon, Ray, Dipankar, Simard,
  Patrice~Y, and Snelson, Ed.
\newblock Counterfactual reasoning and learning systems: the example of
  computational advertising.
\newblock \emph{Journal of Machine Learning Research}, 14\penalty0
  (1):\penalty0 3207--3260, 2013.

\bibitem[Daum{\'e}~III(2009)]{daume2009frustratingly}
Daum{\'e}~III, Hal.
\newblock Frustratingly easy domain adaptation.
\newblock \emph{arXiv preprint arXiv:0907.1815}, 2009.

\bibitem[Dud{\'\i}k et~al.(2011)Dud{\'\i}k, Langford, and Li]{dudik2011doubly}
Dud{\'\i}k, Miroslav, Langford, John, and Li, Lihong.
\newblock Doubly robust policy evaluation and learning.
\newblock In \emph{International Conference on Machine Learning (ICML)}, 2011.

\bibitem[Ganin et~al.(2016)Ganin, Ustinova, Ajakan, Germain, Larochelle,
  Laviolette, Marchand, and Lempitsky]{ganin2016domain}
Ganin, Yaroslav, Ustinova, Evgeniya, Ajakan, Hana, Germain, Pascal, Larochelle,
  Hugo, Laviolette, Fran{\c{c}}ois, Marchand, Mario, and Lempitsky, Victor.
\newblock Domain-adversarial training of neural networks.
\newblock \emph{The Journal of Machine Learning Research}, 17\penalty0 (1),
  2016.

\bibitem[Hill \& Reiter(2006)Hill and Reiter]{hill2006interval}
Hill, Jennifer and Reiter, Jerome~P.
\newblock Interval estimation for treatment effects using propensity score
  matching.
\newblock \emph{Statistics in Medicine}, 25\penalty0 (13):\penalty0 2230--2256,
  2006.

\bibitem[Hill(2011)]{hill2011bayesian}
Hill, Jennifer~L.
\newblock Bayesian nonparametric modeling for causal inference.
\newblock \emph{Journal of Computational and Graphical Statistics}, 20\penalty0
  (1), 2011.

\bibitem[Hoiles \& van~der Schaar(2016)Hoiles and van~der
  Schaar]{hoiles2016bounded}
Hoiles, William and van~der Schaar, Mihaela.
\newblock Bounded off-policy evaluation with missing data for course
  recommendation and curriculum design bounded off-policy evaluation with
  missing data for course recommendation and curriculum design.
\newblock In \emph{International Conference on Machine Learning}, pp.\
  1596--1604, 2016.

\bibitem[Imbens \& Wooldridge(2009)Imbens and Wooldridge]{imbens2009recent}
Imbens, Guido~W and Wooldridge, Jeffrey~M.
\newblock Recent developments in the econometrics of program evaluation.
\newblock \emph{Journal of economic literature}, 47\penalty0 (1):\penalty0
  5--86, 2009.

\bibitem[Jiang \& Li(2016)Jiang and Li]{jiang2015doubly}
Jiang, Nan and Li, Lihong.
\newblock Doubly robust off-policy evaluation for reinforcement learning.
\newblock In \emph{International Conference on Machine Learning (ICML)}, 2016.

\bibitem[Johansson et~al.(2016)Johansson, Shalit, and
  Sontag]{johansson2016learning}
Johansson, Fredrik, Shalit, Uri, and Sontag, David.
\newblock Learning representations for counterfactual inference.
\newblock In \emph{International Conference on Machine Learning (ICML)}, 2016.

\bibitem[Kingma \& Ba(2014)Kingma and Ba]{kingma2014adam}
Kingma, Diederik and Ba, Jimmy.
\newblock Adam: A method for stochastic optimization.
\newblock In \emph{International Conference on Learning Representations}, 2014.

\bibitem[Maurer \& Pontil(2009)Maurer and Pontil]{maurer2009empirical}
Maurer, A and Pontil, M.
\newblock Empirical bernstein bounds and sample variance penalization.
\newblock In \emph{The 22nd Conference on Learning Theory}, 2009.

\bibitem[Pearl(2017)]{pearl2017detecting}
Pearl, Judea.
\newblock Detecting latent heterogeneity.
\newblock \emph{Sociological Methods \& Research}, 46\penalty0 (3):\penalty0
  370--389, 2017.

\bibitem[Rosenbaum \& Rubin(1983)Rosenbaum and Rubin]{rosenbaum1983central}
Rosenbaum, Paul~R and Rubin, Donald~B.
\newblock The central role of the propensity score in observational studies for
  causal effects.
\newblock \emph{Biometrika}, 70\penalty0 (1):\penalty0 41--55, 1983.

\bibitem[Rubin(2005)]{rubin2005causal}
Rubin, Donald~B.
\newblock Causal inference using potential outcomes: Design, modeling,
  decisions.
\newblock \emph{Journal of the American Statistical Association}, 100\penalty0
  (469):\penalty0 322--331, 2005.

\bibitem[Shalit et~al.(2017)Shalit, Johansson, and
  Sontag]{shalit2017estimating}
Shalit, Uri, Johansson, Fredrik, and Sontag, David.
\newblock Estimating individual treatment effect: generalization bounds and
  algorithms.
\newblock In \emph{International Conference on Machine Learning (ICML)}, 2017.

\bibitem[Strehl et~al.(2010)Strehl, Langford, Li, and
  Kakade]{strehl2010learning}
Strehl, Alex, Langford, John, Li, Lihong, and Kakade, Sham~M.
\newblock Learning from logged implicit exploration data.
\newblock In \emph{Advances in Neural Information Processing Systems}, pp.\
  2217--2225, 2010.

\bibitem[Swaminathan \& Joachims(2015{\natexlab{a}})Swaminathan and
  Joachims]{swaminathan15counterfactual}
Swaminathan, Adith and Joachims, Thorsten.
\newblock Counterfactual risk minimization: Learning from logged bandit
  feedback.
\newblock In \emph{Proceedings of the 32nd International Conference on Machine
  Learning}, pp.\  814--823, 2015{\natexlab{a}}.

\bibitem[Swaminathan \& Joachims(2015{\natexlab{b}})Swaminathan and
  Joachims]{swaminathan2015self}
Swaminathan, Adith and Joachims, Thorsten.
\newblock The self-normalized estimator for counterfactual learning.
\newblock In \emph{Advances in Neural Information Processing Systems}, pp.\
  3231--3239, 2015{\natexlab{b}}.

\bibitem[Wager \& Athey(2015)Wager and Athey]{wager2015estimation}
Wager, Stefan and Athey, Susan.
\newblock Estimation and inference of heterogeneous treatment effects using
  random forests.
\newblock \emph{arXiv preprint arXiv:1510.04342}, 2015.

\bibitem[Yoon et~al.(2016)Yoon, Davtyan, and van~der Schaar]{yoon2016discovery}
Yoon, J, Davtyan, C, and van~der Schaar, M.
\newblock Discovery and clinical decision support for personalized healthcare.
\newblock \emph{IEEE journal of biomedical and health informatics}, 2016.

\bibitem[Zhang et~al.(2013)Zhang, Sch{\"o}lkopf, Muandet, and
  Wang]{zhang2013domain}
Zhang, Kun, Sch{\"o}lkopf, Bernhard, Muandet, Krikamol, and Wang, Zhikun.
\newblock Domain adaptation under target and conditional shift.
\newblock In \emph{International Conference on Machine Learning}, pp.\
  819--827, 2013.

\end{thebibliography}

\end{document}